\documentclass[sigconf]{acmart}

\usepackage{microtype}
\usepackage{textcomp}
\usepackage{graphicx}
\usepackage{subfigure}
\usepackage{booktabs} 
\usepackage{subcaption} 
\usepackage{hyperref}
\usepackage{algorithm}
\usepackage{algorithmic}
\usepackage{hyperref}
\usepackage{amsmath}
\usepackage{mathtools}
\usepackage{amsthm}
\usepackage[capitalize,noabbrev]{cleveref}
\theoremstyle{plain}
\newtheorem{theorem}{Theorem}[section]

\newtheorem{corollary}[theorem]{Corollary}
\theoremstyle{definition}
\newtheorem{definition}[theorem]{Definition}

\theoremstyle{remark}

 \AtBeginDocument{%
   }

\copyrightyear{2024}
\acmYear{2024}
\setcopyright{rightsretained}
\acmConference[MOBIHOC '24]{The Twenty-fifth International Symposium on Theory, Algorithmic Foundations, and Protocol Design for Mobile Networks and Mobile Computing}{October 14--17, 2024}{Athens, Greece}
\acmBooktitle{The Twenty-fifth International Symposium on Theory, Algorithmic Foundations, and Protocol Design for Mobile Networks and Mobile Computing (MOBIHOC '24), October 14--17, 2024, Athens, Greece}\acmDOI{10.1145/3641512.3686381}
\acmISBN{979-8-4007-0521-2/24/10}

\begin{document}

\title{Deep Index Policy for Multi-Resource Restless Matching Bandit and Its Application in Multi-Channel Scheduling}

\author{Nida Zamir}
 \email{nidazamir@tamu.edu}
\affiliation{%
 \institution{Texas A\&M University}
 \city{College Station}
 \state{Texas}
 \country{USA}}

\author{I-Hong Hou}
 \email{ihou@tamu.edu}
\affiliation{%
 \institution{Texas A\&M University}
 \city{College Station}
 \state{Texas}
 \country{USA}}

\renewcommand{\shortauthors}{Zamir and Hou}
\begin{abstract}
Scheduling in multi-channel wireless communication system presents formidable challenges in effectively allocating resources. To address these challenges, we {investigate} a multi-resource restless matching bandit (MR-RMB) model for heterogeneous resource systems with an objective of maximizing long-term discounted total rewards while respecting resource constraints. We have also generalized to applications beyond multi-channel wireless. We discuss the Max-Weight Index Matching algorithm, which optimizes resource allocation based on learned {partial index}es. We have derived the policy gradient theorem for index learning. Our main contribution is the introduction of a new Deep Index Policy (DIP), an online learning algorithm tailored for MR-RMB. DIP learns the {{partial index}} by leveraging the policy gradient theorem for restless arms with convoluted and unknown transition kernels of heterogeneous resources. We demonstrate the utility of DIP by evaluating its performance for three different MR-RMB problems. Our simulation results show that DIP indeed learns the {partial index}es efficiently. 
\end{abstract}

\ccsdesc[500]{Theory of computation~Sequential decision making}

\keywords{Online learning, multi-armed bandit, multi-channel scheduling}
\maketitle

\section{Introduction}
\label{Introduction}

Scheduling in wireless communication systems is a critical aspect that plays a pivotal role in optimizing resource utilization and enhancing system performance. The importance of scheduling stems from the inherent limitations and complexities of wireless channels, including limited bandwidth, varying channel conditions, and the presence of interference. Effective scheduling strategies enable efficient allocation of scarce resources, such as bandwidth and transmission slots, to users or applications, thereby maximizing system throughput, minimizing latency, and enhancing overall network performance \cite{sheng2014data}. Scheduling problems have been extensively studied in the literature in various applications such as Age of Information (AoI) \cite{kadota2018scheduling}, Quality of Experience (QoE)~\cite{anand2018whittle}, Mobile Edge Computing (MEC) for wireless Virtual Reality (VR)~\cite{yang2018communication} and opportunistic scheduling problems for downlink data traffic \cite{ aalto2015whittle}. 

Most modern wireless systems, such as those employing orthogonal frequency-division multiple access (OFDMA), have the capability to schedule different users for transmissions on different channels or subcarriers simultaneously. Scheduling problems in such multi-channel wireless networks are especially difficult because a user may experience different channel qualities on different channels. Hence, a controller of the system needs to decide not only whom to schedule, but also which channels to schedule each user on.

In this paper, we {investigate} a multi-resource restless matching bandit (MR-RMB) model to address challenges in multi-channel scheduling problems. This model considers each user as a restless arm whose state, such as queue status, packet delay, recent deliveries, evolves according to both application behaviors and the perceived network services. In each time step, a controller observes the states of all restless arms and then matches each restless arm to a resource (channel), subject to the capacity of the resource. In addition to multi-channel scheduling problems, we also show that the MR-RMB model can be applied to a wide range of applications such as advertisement placements on social media websites and call center scheduling.

{The MR-RMB model extends the restless bandit problem, which is a special case with a single resource. The restless bandit problem is intractable due to its exponentially growing state space}.  To address this challenge, we adopt the technique in Zou et al. \cite{zou2021minimizing}, which proposes a {{partial index}} policy for minimizing AoI in multi-channel wireless networks. This policy calculates a {{partial index}} between each restless arm $n$ and each resource $h$. An important feature of the {partial inde}x is that it only depends on the state of the restless arm $n$ and is independent of the states of all other restless arms. This feature makes calculating the {{partial index}} tractable. We generalize this policy for generic MR-RMB models.

An important limitation of the {partial index} policy is that it requires the complete knowledge of the transition kernel of each restless arm, which, in the case of multi-channel scheduling, consists of the application behaviors and channel qualities of each user, to calculate the {partial index}. In practice, such knowledge may not be available. We propose using deep reinforcement learning for the controller to learn the {partial index}es on the fly without any prior knowledge about the transition kernel of each restless arm. We show that finding the {partial index}es is equivalent to finding the optimal control policy for a family of auxiliary Markov decision processes (MDP). We further derive the policy gradient theorem for the entire family of auxiliary MDPs.

 
Based on the policy gradient theorem, we propose Deep Index Policy (DIP), a new deep reinforcement learning policy that learns the {partial index}es by using actor-critic networks. The utility of DIP is comprehensively evaluated in three different MR-RMB problems, two problem are scheduling problems in multi-channel wireless networks and the third one is an advertisement placement problem in social media websites. For scenarios where {partial index}es can be calculated, DIP indeed converges to {partial index}es efficiently. For other scenarios, DIP significantly outperforms other policies.

 {Our primary contributions are the introduction of DIP, which generalizes the {partial index} policy for generic MR-RMB models, and the development of a learning algorithm that efficiently computes the {partial index} without requiring prior knowledge of the system. Zou et al. \cite{zou2021minimizing} define {partial index} but calculating the {partial index} requires a full knowledge about the system, which may not be available in practice. In contrast, we have a learning algorithm that can efficiently find the {partial index} without knowing anything about the system in advance. Additionally, we advance the policy gradient theorem by incorporating multiple resources, a more complex scenario than the single-resource setting explored by Nakhleh and Hou \cite{nakhleh2022deeptop}. Unlike the restless bandit problem, where the comparison is between activation and idle, our work involves comparing activation on resource $h$ with activation on any other resource. This complexity makes a policy gradient theorem more challenging to drive.}

The rest of the paper is organized as follows: Section~\ref{relatedwork} reviews recent studies on wireless scheduling and restless bandits. Section~\ref{sec:System Model} introduces the model for MR-RMB. Section~\ref{sec:DecompositionIndexPolicy} discusses the {partial index} policy for MR-RMB. Section~\ref{sec:PGTforIndexLearning} establishes the policy gradient theorem for finding {partial index}es. Section~\ref{sec:DIP} introduces the deep index policy that finds {partial index}es through deep reinforcement learning. Section~\ref{sec:simulation} presents the simulation results. Finally, Section~\ref{sec:conclusion} concludes the paper.

\subsection{Notations}
Throughout this paper, we use $\vec{x}$ to denote the vector containing $[x_1, x_2, \dots]$. We use $\vec{x}_{-m}$ to denote the vector $\vec{x}$ without $x_m$, i.e., $[x_1,x_2,\dots, x_{m-1}, x_{m+1}, \dots]$, and use $[\vec{x}_{-m},y]$ to denote the vector $\vec{x}$ with $x_m$ being replaced by $y$, i.e., $[x_1,x_2,\dots, x_{m-1}, y, x_{m+1}, \dots]$

\section{Related Work}
\label{relatedwork}
Scheduling problems have been extensively studied in wireless networks, with many focusing on scenarios with a single shared resource \cite{yang2018communication,kadota2018scheduling}, potentially limiting the generalizability of their findings. However, significant research has also been investigated in multi-channel wireless networks \cite{castro2009qos, cheng2009complexity, kim2014scheduling, wan2016joint, li2018high}. Gopalan, Caramanis and Shakkottai explore an online scheduling algorithm to allocate multiple channels for a queueing system \cite{gopalan2012wireless}. Krishnasamy et al. examine the optimization of energy costs in systems with multiple Base Stations (BSs) \cite{krishnasamy2018augmenting}. Bodas et al. focus on the allocation of multiple channels for the downlink of cellular wireless networks \cite{bodas2013scheduling}. Additionally, Sombabu and Moharir \cite{sombabu2020age}, Xu et al. \cite{xu2016mdp}, Talak, Karaman and Modiano \cite{talak2020improving},  Li and Duan \cite{li2023age} and Fountoulakis et al. \cite{fountoulakis2023scheduling} address the AoI minimization problem to enhance reliability under unstable conditions. All these studies assume that the application behaviors and channel conditions are known in advance.

Some researchers have explored learning approach for scheduling problems with unknown application behaviors or channel conditions. Huang et al. optimize task offloading in the downlink of mobile edge computing systems~\cite{huang2020scheduling}. Leong et al. optimize sensor transmission scheduling for remote state estimation in cyber-physical systems.\cite{leong2020deep}. Zakeri et al. develop transmission scheduling policies to minimize AoI \cite{zakeri2023minimizing}. Naderializadeh et al. study resource management in the downlink of a wireless network with multiple access points (APs) transmitting data to user equipment devices (UE), where each UE can only be served by one AP at a time during scheduling \cite{naderializadeh2021resource}. These studies focus on learning scheduling strategies for a single shared resource. They cannot be easily extended to multi-channel wireless networks.

Scheduling problem has frequently been formulated as a Restless Multi-armed Bandit (RMAB) problem~\cite{xu2019scheduling,chen2022uncertainty}. RMAB is notoriously hard to solve as the size of its state space grows exponentially with the number of arms. The Whittle index is widely embraced as a scalable solution for addressing RMAB problems and has been extensively studied in various applications \cite{mate2020collapsing,tripathi2019whittle,borkar2017opportunistic,nakhleh2022deeptop}. We note that in the RMAB literature there is also a line of work on multi-action bandits considering a single resource \cite{hodge2015asymptotic,killian2021beyond}. {Some of the researcher also considered multiple resources. Wu et al.~\cite{wu2022towards} calculate the index via brute force, which is not feasible for larger state spaces or unknown system behaviors. Simchi-Levi, Sun and Wang~\cite{simchi2023online} do not consider states of the user. Gai, Krishnamachari and Liu~\cite{gai2011combinatorial} consider states, but during the assignment, it does not look at the states. Assignments are made based on the best performance for the long-term reward, and then these assignments are maintained. Therefore, it is only looking for a stationary assignment not based on states.} Zou et al. \cite{zou2021minimizing} study the heterogeneous and unreliable channels to minimize the AoI problem where they have considered multiple dual costs of each resources. In contrast, our work is more general and is applicable to any problem that admits heterogeneous resources.

\section{System Model}\label{sec:System Model}
{In this section, we explore the generic MR-RMB model and its application to various networked systems, including multi-channel wireless networks.}


We consider a system composed of $N$ restless arms (wireless clients), numbered by $n = 1, 2, \dots, N$, $H$ heterogeneous resources (wireless channels), numbered as $h=1, 2, \dots, H$, and a controller (cellular base station) in charge of matching resources to restless arms. Each resource $h$ has a capacity of serving up to $C_h$ restless arms in each time step. To simplify notation, we also introduce a \emph{Null} resource $h=0$ with infinite capacity, i.e., $C_0=\infty$.  In each time step $t$, a controller observes the state of each arm $n$, denoted by $s_{n,t}\in S_n$, and then chooses a resource $a_{n,t}\in  \mathbb{A}$, where  $\mathbb{A}=\{0,1,2,\dots, H\}$, to serve each arm $n$ such that at most $c_h$ arms are served by resource $h$. If $a_{n,t} = 0$, then arm $n$ is not served by any resource in time step $t$. After being served by resource $a_{n,t}$, arm $n$ generates a reward $r_{n,t}$ and changes its state to $s_{n,t+1}$ in the next time step. We assume that the reward and state evolution of each arm follows a MDP
. Specifically, when arm $n$ is in state $s_n$ and is being served by resource $a_n$, then it generates a random reward with unknown mean $R_n(s_n, a_n)$ and changes its state to state $s'_n$ with unknown probability $P_n(s'_n|s_n,a_n)$. 


There are many real-world scenarios where the controller needs to decide not only which restless arms to serve, but also which resource to use for each arm. We demonstrate three examples of such systems below


\textbf{Example 1:} In multi-channel wireless systems, a base station (BS) serves $N$ data flows (restless arms) with H heterogeneous channels (resources) and $p_{n,h}$ represents the channel quality between user $n$ and channel $h$. The model for state and reward of a data flow is specified by its application and can be defined based on, for example, queue status, packet delay, data freshness, etc. Each data flow experiences different $p_{n,h}$ on different channels. In each time step, the BS chooses data flows to transmit over each channel.
    
\textbf{Example 2:}  In social media website advertisement, the website has three places to display advertisements: the overhead banner, the sidebar, and the newsfeed. Hence, $H=3$ and $C_h$ represents the number of spots in each place. Each restless arm is an advertisement whose state includes the time and place it was last displayed. In each time step, the website administrator determines whether and where to display each advertisement. The reward of each advertisement, measured in terms of the click-through rate, depends on the state of the advertisement and the place that it is displayed in.

\textbf{Example 3:} Within a Maternal and Child Health Program, call center services employ live voice scheduling to provide timely preventive care information to expecting and new mothers throughout pregnancy and up to one year post-birth. Call center has $H$ callers available, each caller has a different language expertise and can call up to $C_h$ expecting/new mothers each week. Each expecting/new mother is a restless arm with her own language preference, and her state is either engaged in preventive care or not. Each time step, the call center determines which caller to contact to increase her engagement in preventive care. The effectiveness of a call depends on the mother's state and whether the caller's language expertise matches the mother's. A recent study \cite{verma2023restless} has studied the special case when all callers have the same language expertise. 

\sloppy {The controller employs a matching policy $\vec{\pi}$ to match arms to resources. Let $\vec{s}_t:=[s_{1,t}, s_{2,t},\dots, s_{N,t}]$ and $\vec{a}_t:=[a_{1,t}, a_{2,t},\dots,a_{N,t}]$, then the matching policy can be viewed as a function $\vec{\pi}$ that determines $\vec{a}_t=\vec{\pi}(\vec{s}_t):=[\pi_{1}(\vec{s}_t), \pi_{2}(\vec{s}_t),\dots]$. We evaluate $\vec{\pi}$ by its long-term discounted total reward. Specifically, let $\beta$ be the discount factor and let $\mathbb{I}(\cdot)$ be the indicator function, then the controller's goal is to find the optimal $\vec{\pi}$ for the following optimization problem:
\begin{align}
    \mbox{\textbf{SYSTEM:}} \hspace{10pt} \max_{\vec{\pi}}\mbox{ } & E[\sum_{t=0}^\infty \sum_{n=1}^N \beta^tR_n\big(s_{n,t}, \pi_n(\vec{s}_{t})\big)] \label{equation:SYSTEM1}\\
    \mbox{s.t. } & \sum_{n=1}^N \mathbb{I}\big(\pi_n(\vec{s})=h\big)\leq C_h, \forall \vec{s}, h, \label{equation:SYSTEM2}\\
    \mbox{and } & \pi_n(\vec{s})\in\{0,1,2,\dots, H\}, \forall \vec{s},n.\label{equation:SYSTEM3}
\end{align}}

\section{Preliminary: Decomposition and Index Policy}\label {sec:DecompositionIndexPolicy}
 
The problem \textbf{SYSTEM} is intractable to solve because the state space of $\vec{s}$ is the product of the state spaces of $s_n$ and its size increases exponentially with $N$. A recent study \cite{zou2021minimizing} has employed a decomposition technique to develop an index policy for the problem of minimizing AoI in multi-channel wireless systems. In this section, we generalizes its result for generic MR-RMB problems.

\subsection{Lagrange Decomposition}\label {sec:LangrangeDecomposition}
 To simplify the \textbf{SYSTEM} problem, we first relax the per-$\vec{s}$ constraint (\ref{equation:SYSTEM2}) to an average constraint:
\begin{align}
    E\left[\sum_{t=0}^\infty\sum_{n=1}^N \beta^t\mathbb{I}(\pi_n(\vec{s_t})=h)\right]\leq \sum_{t=0}^\infty \beta^tC_h=\frac{C_h}{1-\beta}, \forall h, \label{equation:RelaxedSYSTEM}
\end{align}
and then introduce a Lagrange multiplier $\vec{\lambda}:=[\lambda_1, \lambda_2, \dots, \lambda_H]$ for this relaxed constraint. We can view $\vec{\lambda}$ as a shadow price so that the controller needs to pay $\lambda_h$ for every arm that is matched to resource $h$. The Lagrangian of the relaxed problem is then
\begin{align}
\textbf{Lagr($\vec{\lambda}$):} \hspace{1pt} \max_{\vec{\pi}}\mbox{ }  &E\left[\sum_{t=0}^\infty \sum_{n=1}^N \beta^t\Big(R_n\big(s_{n,t}, \pi_n(\vec{s}_t)\big)-\lambda_{\pi_n(\vec{s}_t)}\Big)\right],\nonumber \\&\mbox{s.t. }\pi_n(\vec{s})\in\{0,1,2,\dots, H\}, \forall \vec{s},n. \label{equation:Lagrange1}
\end{align}
An important feature of \textbf{Lagr($\vec{\lambda}$)} is that it can be decomposed into $N$ subproblems, one for each arm. Specifically, let $\sigma_{n,\vec{\lambda}}(\cdot)$ be the optimal solution to the \textbf{Arm$_n(\vec{\lambda})$} problem,
\begin{align}
    \mbox{\textbf{Arm$_n(\vec{\lambda})$:}} \hspace{1pt} \max_{\sigma_n}\mbox{ } & E\left[\sum_{t=0}^\infty \beta^t\Big(R_n\big(s_{n,t}, \sigma_n(s_{n,t})\big)-\lambda_{\sigma_n(s_{n,t})}\Big)\right],\nonumber \\
    &\mbox{s.t. }\sigma_n(s_n)\in\{0,1,2,\dots, H\}, \forall s_n, \label{equation:Arm1}
\end{align}
then choosing $\pi_n(\vec{s})= \sigma^{*}_{n,\vec{\lambda}}(s_n)$ solves \textbf{Lagr($\vec{\lambda}$)}.

\sloppy There are three important advantages of the above decomposition. First, this decomposition addresses the curse of dimensionality since each \textbf{Arm$_n(\vec{\lambda})$} problem only involves the state space of arm $n$. Second, the decomposition preserves optimality. Specifically, when $\vec{\lambda}$ is chosen as the minimizer of $L(\vec{\lambda}):=\min_{\vec{\lambda}}\max_{\vec{\pi}}E[\sum_{t=0}^\infty \sum_{n=1}^N \beta^t\Big(R_n\big(s_{n,t}, \pi_n(\vec{s}_t)\big)-\lambda_{\pi_n(\vec{s}_t)}\Big)]$, then the optimal solution to \textbf{Arm$_n(\vec{\lambda})$} problem is also the optimal solution to \textbf{SYSTEM} with the relaxed constraint Eq. (\ref{equation:RelaxedSYSTEM}). Finally, the minimizer of $L(\vec{\lambda})$ can be found iteratively through a simple gradient algorithm: 
\begin{align}
    \lambda_h^{(k+1)}=\Big[\lambda_h^{(k)}+\rho_k\Big(E[\sum_{t=0}^\infty\sum_{n=1}^N \beta^t\mathbb{I}(\sigma_{n,\lambda_n}&(s_{n,t})=h)]-\frac{C_h}{1-\beta}\Big)\Big]^+,
\end{align}
where $\lambda_h^{(k)}$ is the value of $\lambda_h$ in the $k-$th iteration, $\rho_k$ is a properly chosen step size, and $[x]^+:=\max\{x,0\}$.

\subsection{{Partial Index} and Max-Weight Index Matching}\label{sec:PWInMWIM}

One important drawback of Lagrange decomposition is that it needs to relax the per-$\vec{s}$ constraint (\ref{equation:SYSTEM2}). For the special case when there is only one resource, i.e., $H=1$, the Whittle index overcomes this drawback by producing a policy that both satisfies the per-$\vec{s}$ constraint (\ref{equation:SYSTEM2}) and is asymptotically optimal under some mild conditions. Following the approach in Zou et al. \cite{zou2021minimizing}, one can define a \emph{{partial index}} for each arm $n$, each state $s_n$, and each resource $h$. For a given arm $n$, it defines a \emph{{partial index}} for each state $s_n$ of $n$ and each resource $h$, denoted by $w_{n,h}(s_n,\vec{\lambda}_{-h})$, as the following:

\begin{definition}\label{pw} [{Partial Index}]
    Given an arm $n$ and a Lagrange multiplier $\vec{\lambda}$, the {partial index} for state $s_n$ and resource $h$ is defined as
    \begin{align}
        w_{n,h}(s_n,\vec{\lambda}_{-h}):= sup\{y|\sigma^{*}_{n,[\vec{\lambda}_{-h},y]}(s_n)=h\}.
    \end{align}
\end{definition}

Intuitively, $w_{n,h}(s_n,\vec{\lambda}_{-h})$ can be viewed as the highest shadow price that arm $n$ is willing to pay to be matched to resource $h$, instead of any other resources, when its state is $s_n$. Hence, arm $n$ should prefer resource $h$ over others as long as $\lambda_h\leq w_{n,h}(s_n,\vec{\lambda}_{-h})$. Arm $n$ is said to be \emph{indexable} when it indeed exhibits such a behavior:

\begin{definition}\label{index} [Indexability]
    An arm $n$ is indexable if, for any $s_n$, $\vec{\lambda}$, and $h$, we have, for any $y\leq w_{n,h}(s_n,\vec{\lambda}_{-h})$, $\sigma^{*}_{n,[\vec{\lambda}_{-h},y]}(s_n)=h$.
\end{definition}
      \begin{algorithm}[tb]
        \caption{Max-Weight Index Matching} \label{alg:max matching}
        \begin{algorithmic}
          \STATE Initialize $\vec{\lambda}$
          \FOR{t=0, 1, 2, \dots}
          \STATE Calculate $w_{n,h}(s_{n,t},\vec{\lambda}_{-h})$ for all $n, h$
          \STATE Create a bipartite graph with $N$ arm nodes and $H+1$ resource nodes
          \STATE Add an edge with weight $w_{n,h}(s_{n,t},\vec{\lambda}_{-h})$ between arm node $n$ and resource node $h$
          \STATE Find the max-weight matching and match arms to resources accordingly
          \STATE $\lambda_h\leftarrow \Big[\lambda_h+\rho_t\Big(\sum_n\mathbb{I}(w_{n,h}(s_{n,t}\vec{\lambda}_{-h})>\lambda_h)-C_h\Big)\Big]^+,\forall h$
          \ENDFOR
        \end{algorithmic}
      \end{algorithm}

  
We now discuss the matching algorithm. Given $\vec{\lambda}$, the algorithm first calculates the {partial index} for each $n$, $s_n$, and $h$. In each time step $t$, the algorithm creates a bipartite graph with $N+H+1$ nodes, such that each arm and each resource is a node. There is an edge between each arm $n$ and each resource $h$ with weight $w_{n,h}(s_{n,t},\vec{\lambda}_{-h})$. The algorithm then finds the max-weight matching between the nodes, with the constraint that each resource $h$ can be matched to at most $C_h$ arms. We describe the algorithm along with a simplified update rule for $\vec{\lambda}$ in Alg.~\ref{alg:max matching}. Zou et al. \cite{zou2021minimizing} has proved that Alg.~\ref{alg:max matching} is asymptotically optimal for a specific wireless scheduling problem of AoI minimization. However, to calculate the {partial index}, one needs to know the transition kernels of each restless arms. In the next two sections, we will introduce an online reinforcement learning algorithm that learns the {partial index} for restless arms with convoluted and unknown transition kernels.

\section{Policy Gradient Theorem for Index Learning}\label{sec:PGTforIndexLearning}



{In this section, we study the fundamental properties of the {partial index} $w_{n,h}(s_n,\vec{\lambda}_{-h})$. We will show that $w_{n,h}(s_{n},\vec{\lambda}_{-h})$ is the optimal solution to an auxiliary MDP. Nakhleh and Hou \cite{nakhleh2022deeptop} has derived the policy gradient theorem for finding the Whittle index, which is equivalent to the special case of $H=1$ in our paper. We expand Nakhleh and Hou \cite{nakhleh2022deeptop} to address the more general case of multiple resources and derive the corresponding policy gradient theorem.}

{Throughout this section, we will focus on studying $w_{n,h}(s_{n},\vec{\lambda}_{-h})$ for fixed $n$, $h$, and $\vec{\lambda}_{-h}$.
We also assume that the {partial index}es for all other resources, that is, $w_{n,h^{'}}(s_{n},\vec{\lambda}_{-h^{'}})$ for all $h^{'} \neq h$, are known and given. We drop the subscript $n$ from all notations for better clarity.}

\sloppy{Let $w^{\phi_h}_{h}(s)$ be a function with parameter vector $\phi_h$ that predicts the {partial index}  $w_{h}(s,\vec{\lambda}_{-h})$. For a given resource $h$ of an \textbf{Arm$_n([\vec{\lambda}_{-h},\lambda_h])$} problem, we can construct the following policy, which we denote by $\sigma_{\lambda_h}^{\phi_h}(s)$:  First, the policy compares  $w^{\phi_h}_{h}(s)$ with $\lambda_h$  and sets $\sigma_{\lambda_h}^{\phi_h}(s)=h$ if $w^{\phi_h}_{h}(s)\geq \lambda_h$.  Second, if $w^{\phi_h}_{h}(s)\leq \lambda_h$, then the policy finds the largest $h^{'}\neq h $ with $w_{h^{'}}(s,\vec{\lambda}_{-h^{'}})\geq \lambda_{h^{'}}$ and sets $\sigma_{\lambda_h}^{\phi_h}(s)=h^{'}$. Finally, if $w^{\phi_{h}}_{h}(s)<\lambda_{h}$ and $w_{h^{'}}(s,\vec{\lambda}_{-h^{'}})< \lambda_{h^{'}}$ for all $h^{'} \neq h$, then policy sets $\sigma_{\lambda_h}^{\phi_h}(s)=0$ .}

\sloppy {We can now define the state-action function of applying $\sigma_{\lambda_h}^{\phi_h}(s)$ to \textbf{Arm$_n([\vec{\lambda}_{-h},\lambda_h])$} problem, which we  denote by  $\mathcal{Q}^{\phi_h}(s,a,\lambda_h)$. The Bellman equation of $\mathcal{Q}^{\phi_h}(s,a,\lambda_h)$ is defined as:}
   \begin{align} \label{eq:qFunction}
\mathcal{Q}^{\phi_h}(s,a,\lambda_h) = &R(s,a)- \lambda_{a} \nonumber \\
+&\beta \sum_{s^{'}\in S} P(s^{'}|s,a) \mathcal{Q}^{\phi_h}(s^{'},\sigma_{\lambda_h}^{\phi_h}(s^{'}),\lambda_h) )
\end{align}

We then have the following Corollary:

\begin{corollary}\label{cor:1} If arm $n$ is indexable, then setting $w^{\phi_h}_{h}(s)$ to be its {partial index} $w_{h}(s,\vec{\lambda}_{-h})$ maximizes $\mathcal{Q}^{\phi_h}(s,a,\lambda_h)$ for any $\lambda_h$.
 \end{corollary}
 \begin{proof}
 If $w^{\phi_h}_{h}(s)\equiv w_h(s,\vec{\lambda}_{-h})$, then $\sigma_{\lambda_h}^{\phi_h}(s)=g$ if and only if $\lambda_{g}\leq w_{g}(s,\vec{\lambda}_{-g})$ for any $g$. By Definition \ref{index}, $\sigma_{\lambda_h}^{\phi_h}(s)=\sigma^{*}_{[\vec{\lambda}_{-h},\lambda_h]}(s)$ for all $\lambda_h$ and $s$. Since $\sigma^{*}_{[\vec{\lambda}_{-h},\lambda_h]}(s)$ is the optimal solution to the \textbf{Arm$_n([\vec{\lambda}_{-h},\lambda_h])$} problem, setting $w^{\phi_h}_{h}(s)\equiv w_h(s,\vec{\lambda}_{-h})$ maximizes $\mathcal{Q}^{\phi_h}(s,a,\lambda_h)$.
\end{proof}

\begin{figure}
\centering
        \includegraphics[width=0.8\linewidth]{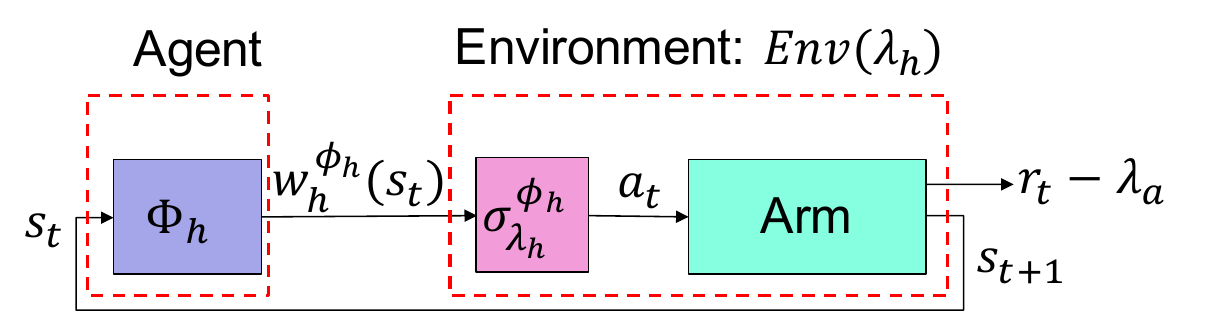}
    \caption{An illustration of Corollary~\ref{cor:1}}
    \label{fig:rl_formulation}
\end{figure}

To understand the implication of Corollary~\ref{cor:1}, consider the reinforcement learning problem as illustrated in Fig.~\ref{fig:rl_formulation} that contains an agent and an environment called $Env(\lambda_h)$. In each time step, the agent observes the state $s_t$ and chooses a real number as its control decision. When $Env(\lambda_h)$ receives the control decision, it first treats the control decision as $w^{\phi_h}_h(s_t)$ and employs $\sigma^{\phi_h}_{\lambda_h}$ to determine $a_t$. It then feeds $a_t$ to the restless arm to generate the next state $s_{t+1}$ and the net reward $r_t-\lambda_{a_t}$. The agent's goal is to find a control policy that can maximize the long-term average net reward, $\sum_{t=0}^\infty \beta^tE[r_t-\lambda_{a_t}]$. Corollary~\ref{cor:1} states that the {partial index} $w_h(s, \vec{\lambda}_{-h})$ is the optimal control policy for $Env(\lambda_h)$ for all $\lambda_h$. 

Based on Corollary \ref{cor:1}, we define the objective function for learning $\phi_{h}$ as 
\begin{align} \label{eq:jFunction}
J^{{\phi_{h}}} = \sum_{s \in S} \int_{\lambda_h=-M}^{\lambda_h=+M} \mathcal{Q}^{\phi_h}(s,\sigma_{\lambda_h}^{\phi_h}(s),\lambda_h)  d\lambda_h,
\end{align}
where $M$ is a sufficiently large constant such that ${\lambda_h} \in [-M,+M]$. If $J^{{\phi_{h}}}$ has a unique maximizer, then maximizing $J^{{\phi_{h}}}$ is equivalent to finding $\phi_h$ such that $w_h^{\phi_h}(\cdot)$ is the {partial index}. We therefore seek to find the {partial index} by finding $\phi_h$ that maximizes $J^{{\phi_{h}}}$. {We will find $\nabla_{\phi_{h}}J^{{\phi_{h}}}$ as shown in Theorem~\ref{thm:1} below}.


\begin{theorem}\label{thm:1} 
Given the parameter vector ${\phi}_h$, if all states $s \in S$ have distinct values of $w^{\phi_h}_{h}(s)$, then the gradient of the objective function $J^{{\phi_{h}}}$ with respect to the parameter vector ${\phi}_{h}$ is given by:
   \begin{align}\label{eq:jGradient} 
    \nabla _{\phi_{h}}J^{{\phi_{h}}}  = &  \sum_{s \in S}\left[\mathcal{Q}^{\phi_h}(s,h, w^{\phi_h}_{h}(s))\right.\left.-\mathcal{Q}^{\phi_h}(s,\sigma'_{\lambda_h}(s), w^{\phi_h}_{h}(s))\right]\nonumber \\
  &\nabla _{\phi_{h}}  w^{\phi_h}_{h}(s),  
  \end{align}
  where
\begin{align}
  &\sigma'_{\lambda_h}(s)\nonumber\\ 
  =&
    \begin{cases}
      0, &\text{if $w_{h^{'}}(s,\vec{\lambda}_{-h^{'}})< \lambda_{h^{'}}, \forall h^{'} \neq h$,}\\
      \max \{h'| w_{h^{'}}(s,\vec{\lambda}_{-h^{'}})\geq \lambda_{h^{'}}\},& \text{otherwise.}
    \end{cases}       
\end{align}
 \end{theorem}

\begin{proof}
Taking the gradient of Eq. (\ref{eq:jFunction}) yields
\begin{align} \label{eq:gradientJa}
\nabla _{\phi_h}J^{{\phi_h}} &= \nabla _{\phi_h}\sum_{s \in S} \int_{\lambda_h=-M}^{\lambda_h=+M} \mathcal{Q}^{\phi_h}(s,\sigma_{\lambda_h}^{\phi_h}(s),\lambda_h) d\lambda_h.
\end{align}
   Renumber all states in $S$ such that $w_{h}^{\phi_h}( s_{|S|}) $ $\textgreater$  $w_{ h}^{\phi_h}(s_{|S|-1}) $  $\textgreater$  . . .  $w_{h}^{\phi_h}( s_{1})$. Let $\mathbb{M}^0$ = $-M$, $\mathbb{M}^k$ = $w_{ h}^{\phi_h}(s_{k})$, for each $1\leq k\leq |S|$, and $\mathbb{M}^{|S|+1}$ = $+M$. Also let $S_k$ be the subset of states $\{ s_k, s_{k+1}, . . . , s_{|S|}\}$, for each $1\leq k\leq |S|$. Then, for any $k$, $S_k$ is the subset of states with $w_h^{\phi_h}(s)\geq \mathbb{M}^k$. Hence, for any $\lambda_h$ in the interval $(\mathbb{M}^k, \mathbb{M}^{k+1})$, we have
\begin{align}
  {\sigma}_{\lambda_h}^{\phi_h}(s)=
    \begin{cases}
      h, &\text{if $s \in S_{k+1}$,}\\
      \sigma'_{\lambda_h}(s),& \text{otherwise.}
    \end{cases}       
\end{align}

Define the policy $\hat{\sigma}^{k+1}_{\lambda_h}(s)$ as the policy that chooses action $h$ when $\mathbb{I}{{(s \in S_{k+1})}}=1$  and $\sigma'_{\lambda_h}(s)$ otherwise. Let $\hat{\mathcal{Q}}_{k+1}(s,a,\lambda_h)$ be the state-action function of applying $\hat{\sigma}^{k+1}_{\lambda_h}(s)$. Then, for any ${\lambda}_h \in (\mathbb{M}^k, \mathbb{M}^{k+1})$, ${\sigma}_{\lambda_h}^{\phi_h}(s)=\hat{\sigma}^{k+1}_{\lambda_h}(s)$ for all $s$, and therefore $\mathcal{Q}^{\phi_h}(s,a,\lambda_h)=\hat{\mathcal{Q}}_{k+1}(s,a,\lambda_h)$ for all $s$ and $a$. We can now rewrite Eq. (\ref{eq:gradientJa}) as
   \begin{align} \label{eq:gradientJc}
    \nabla _{\phi_h}J^{{\phi_h}}=   \sum_{k=0}^{|S|}\sum_{s \in S} \nabla _{\phi_h} \int_{\lambda_h=\mathbb{M}^{k}}^{\lambda_h=\mathbb{M}^{k+1}}\hat{\mathcal{Q}}_{k+1}(s,\hat{\sigma}^{k+1}_{\lambda_h}(s),\lambda_h) d\lambda_h.\end{align}


    Applying Leibniz integral rule and we have:
    \begin{align} \label{eq:gradientJe}
    \nabla _{\phi_h}J^{{\phi_h}}=  & \sum_{k=0}^{|S|} \sum_{s \in S}\left[\hat{\mathcal{Q}}_{k+1}(s,\hat{\sigma}^{k+1}_{\lambda_h}(s),\mathbb{M}^{k+1}) \right. \nabla _{\phi_h}  \mathbb{M}^{k+1}\nonumber \\
  &\left.-\hat{\mathcal{Q}}_{k+1}(s,\hat{\sigma}^{k+1}_{\lambda_h}(s),\mathbb{M}^{k})\nabla _{\phi_h}  \mathbb{M}^{k} \right] \nonumber \\
  &+ \sum_{k=0}^{|S|} \sum_{s \in S} \int_{\lambda_h=\mathbb{M}^{k}}^{\lambda_h=\mathbb{M}^{k+1}}\nabla _{\phi_h} \hat{\mathcal{Q}}_{k+1}(s,\hat{\sigma}^{k+1}_{\lambda_h}(s),\lambda_h)d\lambda_h. \end{align}

Since $\hat{\mathcal{Q}}_{k+1}(s,\hat{\sigma}^{k+1}_{\lambda_h}(s),\lambda_h)$ is constant with respect to $\phi_h$, $\nabla _{\phi_h} \hat{\mathcal{Q}}_{k+1}(s,\hat{\sigma}^{k+1}_{\lambda_h}(s),\lambda_h)=0$. Moreover, for any $k$, $\hat{\sigma}^{k+1}_{\lambda_h}(s)=\hat{\sigma}^{k}_{\lambda_h}(s)$ for all $s$ except $s_k$. Hence, we can further simplify Eq.~(\ref{eq:gradientJe}) and obtain

 \begin{align} 
 &\nabla _{\phi_h}J^{{\phi_h}}\nonumber\\
  =&\sum_{k=0}^{|S|} \sum_{s \in S}\left[\mathcal{Q}^{\phi_h}(s,\hat{\sigma}^{k+1}_{\lambda_h}(s),\mathbb{M}^{k+1})\nabla _{\phi_h}  \mathbb{M}^{k+1} \right.\nonumber \\
  &\left.- \mathcal{Q}^{\phi_h}(s,\hat{\sigma}^{k+1}_{\lambda_h}(s),\mathbb{M}^{k})\nabla _{\phi_h}  \mathbb{M}^{k} \right] \nonumber \\
  =&\sum_{k=1}^{|S| }\left[\mathcal{Q}^{\phi_h}(s_k,\hat{\sigma}^{k}_{\lambda_h}(s_k),\mathbb{M}^{k})-\mathcal{Q}^{\phi_h}(s_k,\hat{\sigma}^{k+1}_{\lambda_h}(s_k),\mathbb{M}^{k})\right]\nabla _{\phi_h}  \mathbb{M}^{k}  \nonumber  \\
  =& \sum_{k=1}^{|S| }\left[\mathcal{Q}^{\phi_h}(s_k,h,w_{h}^{\phi_h}(s_k))-\mathcal{Q}^{\phi_h}(s_k,\sigma'_{\lambda_h}(s_k),w_{h}^{\phi_h}(s_k))\right]\nabla _{\phi_h} w_{h}^{\phi_h}(s_k)\nonumber \\
  =& \sum_{s \in S}\left[\mathcal{Q}^{\phi_h}(s,h,w_{h}^{\phi_h}(s))\right. \left.-\mathcal{Q}^{\phi_h}(s,\sigma'_{\lambda_h}(s),w_{h}^{\phi_h}(s))\right]\nabla _{\phi_h} w_{h}^{\phi_h}(s)  \nonumber 
  \end{align}


This completes the proof.
\end{proof}

\section{Deep Index Policy}\label{sec:DIP}

In this section, we introduce DIP, a new Deep Index Policy, for online learning of the {partial index}es by leveraging the policy gradient theorem for index learning (Thm. \ref{thm:1}). For each restless arm $n$, DIP maintains $H$ actor networks, parameterized by $\phi_{n,1}, \phi_{n,2},\dots$, and a critic network, parameterized by $\theta_{n}$. Each actor network $\phi_{n,h}$ takes $(s_n, \vec{\lambda}_{-h})$ as input and produces a number $w^{\phi_{n,h}}_{n,h}(s_n,\vec{\lambda}_{-h})$ as its output. DIP aims to train $\phi_{n,h}$ such that the actor network predicts the {partial index}, that is, $w^{\phi_{n,h}}_{n,h}(s_n,\vec{\lambda}_{-h})\approx w_{n,h}(s_n,\vec{\lambda}_{-h})$, for all $n$ and $h$. Each critic network $\theta_n$ takes $(s_n, a_n, \vec{\lambda})$ as input and produces a number $\mathcal{Q}_n^{\theta_n}(s_n, a_n, \vec{\lambda})$. DIP aims to train $\theta_n$ such that the critic network predicts the state-action function of applying the optimal policy to the \textbf{Arm}$_n(\vec{\lambda})$ problem, which we denote by $\mathcal{Q}^*_n(s_n, a_n, \vec{\lambda})$. The Bellman equation of $Q^*_n(s_n, a_n, \vec{\lambda})$ is
\begin{align} \label{eq:Q^*forArm_n}
\mathcal{Q}^*_n(s_n,a_n,\vec{\lambda}) = &R_n(s_n,a_n)- \lambda_{a_n} \nonumber \\
+&\beta \sum_{s'_n\in S} P(s'_n|s_n,a_n) \max_{a'}\mathcal{Q}^{*}(s'_n,a',\lambda_h).
\end{align}
DIP also maintains a target critic network with parameters $\theta_{n}'$ for each $n$. The target critic networks are updated at a slower rate compared to critic parameters $\theta_{n}$ to ensure stability and improve the training process by providing consistent target values for the critic network \cite{lillicrap2015continuous}. 

We first provide an overview of the procedure of DIP. In each time step, DIP employs a exploration-exploitation policy similar to the $\epsilon-$greedy policy. With probability $\epsilon$, DIP randomly assigns restless arms to resources for the purpose of exploration. With probability $1-\epsilon$, DIP employs Max-Weight Index Matching where the weight between restless arm $n$ and resource $h$ is set to be $w^{\phi_{n,h}}_{n,h}(s_{n,t},\vec{\lambda}_{-h})$. DIP observes the actions, rewards, and state transitions of all restless arms and store them in a replay buffer. DIP then updates the value of $\vec{\lambda}$ by $\lambda_h\leftarrow \Big[\lambda_h+\rho_t\Big(\sum_n\mathbb{I}(w^{\phi_{n,h}}_{n,h}(s_{n,t}\vec{\lambda}_{-h})>\lambda_h)-C_h\Big)\Big]^+,$ for all $h$. Finally, DIP updates all actor networks, critic networks, and target critic networks. {DIP is based on Deep Deterministic Policy Gradient (DDPG), an off-policy reinforcement learning algorithm, to learn the {partial index}es. This off-policy nature is crucial because it allows the learning of the partial index that optimizes the fictitious policy described in Thm. \ref{thm:1}, while the max-weight matching policy is being executed.} Alg.~\ref{alg:dip} describes the detailed procedure of DIP. 

      \begin{algorithm}[tb]
        \caption{Deep Index Learning} \label{alg:dip}
        \begin{algorithmic}
          \STATE Initialize $\vec{\lambda}$, $\phi_{n,h}$, and $\theta_n$
          \STATE Initialize a replay memory $\mathcal{M}_n$ for each $n$
          \STATE $\theta'_n\leftarrow\theta_n$
          \FOR{t=0, 1, 2, \dots}
          \STATE $x\sim U(0,1)$
          \IF{$x<\epsilon$}
          \STATE Randomly match arms to resources
          \ELSE
          \STATE Create a bipartite graph with $N$ arm nodes and $H+1$ resource nodes
          \STATE Add an edge with weight $w^{\phi_{n,h}}_{n,h}(s_{n,t},\vec{\lambda}_{-h})$ between arm node $n$ and resource node $h$
          \STATE Find the max-weight matching and match arms to resources accordingly
          \STATE $\lambda_h\leftarrow \Big[\lambda_h+\rho_t\Big(\sum_n\mathbb{I}(w^{\phi_{n,h}}_{n,h}(s_{n,t}\vec{\lambda}_{-h})>\lambda_h)-C_h\Big)\Big]^+,\forall h$
          \ENDIF
          \STATE Save $(s_{n,t}, a_{n,t}, r_{n,t}, s_{n,t+1})$ to $\mathcal{M}_n$
          \STATE Run Alg.~\ref{alg:MR-RMAB} to update all neural networks
          \ENDFOR
        \end{algorithmic}
      \end{algorithm}

We now discuss how DIP updates all neural networks. For each $n$, DIP randomly samples a batch of transitions $(s_{n,t},a_{n,t},r_{n,t},s_{n,t+1})$ from the replay buffer and attach a randomly selected $\vec{\lambda}\in[-M,+M]^H$ to each transition. For each transition, DIP calculates 
\begin{align}\label{eq:actor_gradient}
&[\mathcal{Q}_n^{\theta_n}(s_{n,t}, a_{n,t}, [\vec{\lambda}_{-a_{n,t}}, w^{\phi_{n,a_{n,t}}}_{n,a_{n,t}}(s_{n,t},\vec{\lambda}_{-a_{n,t}})])\nonumber\\
-&\mathcal{Q}_n^{\theta_n}(s_{n,t}, \sigma'_{\lambda_{a_{n,t}}}(s), [\vec{\lambda}_{-a_{n,t}}, w^{\phi_{n,a_{n,t}}}_{n,a_{n,t}}(s_{n,t},\vec{\lambda}_{-a_{n,t}})])]\nonumber\\
\times&\nabla _{\phi_{n,a_{n,t}}}  w^{\phi_{n,a_{n,t}}}_{n,a_{n,t}}(s_{n,t},\vec{\lambda}_{-a_{n,t}})
\end{align}
as an approximation of Eq. (\ref{eq:jGradient}) and uses it to update the actor networks $\phi_{n, a_{n,t}}$.

DIP then uses the same batch of transitions and $\vec{\lambda}$ to update the critic network. Based on the Bellman equation of $Q^*_n(s_n,a_n, \vec{\lambda})$ (Eq. (\ref{eq:Q^*forArm_n})), DIP defines the loss function of the critic function as
\begin{align}\label{eq:lossFunction}
\mathcal{L}_n^{{\theta}_n} =& \mathbb{E} \left[ \left(\mathcal{Q}_n^{{\theta}_n}(s_{n,t},{a_{n,t}},{\vec{\lambda}})-r_{n,t} + {\lambda}_{a_{n,t}}\right.\right. \nonumber \\ & \left. \left. -\beta \max _{a'}
\mathcal{Q}_n^{{\theta}'_n}(s_{n,t+1},a',{\vec{\lambda}}) \right)^2\right].
\end{align}

DIP then uses the each transition and $\vec{\lambda}$ from the batch to estimate the gradient of the loss function by
\begin{align}\label{eq:gradlossFunction}
 &2 \left(\mathcal{Q}_n^{{\theta}_n}(s_{n,t},{a_{n,t}},\vec{\lambda})-r_{n,t} + {\lambda}_{a_{n,t}} \right. \nonumber \\ 
 &  \left.-\beta \max _{a'}
\mathcal{Q}_n^{{\theta}^{'}_n}(s_{n,t+1},a',\vec{\lambda}) \right) \nabla _{\theta_n}\mathcal{Q}_n^{{\theta}_n}(s_{n,t},{a_{n,t}},\vec{\lambda}).
\end{align}
and to update $\theta_n$. Finally, DIP soft updates the target critic network by $\theta'_{n}\leftarrow\tau \theta_{n}+(1-\tau)\theta'_{n}$. {The training complexity scales with the product of the number of arms and the number of resources i.e., \(O(N \times H)\), as each (arm, resource) pair is trained independently.}

    \begin{algorithm}[tb]
        \caption{Neural Networks Updates} \label{alg:MR-RMAB}
        \begin{algorithmic}
          \FOR{n= 1, 2, \dots, N}
          \STATE Sample a mini batch $B$ transitions $(s_{n,{t_k}},a_{n,{t_k}},r_{n,{t_k}},s_{n,t_{k+1}})$, for $1 \leq k \leq B$ from memory $\mathcal{M}_n$.
          \STATE Randomly sample $B$ different shadow price  [$\vec{\lambda}_1$, $\vec{\lambda}_2$, $\ldots$ $\vec{\lambda}_B$ ] from $[-M,M]^H$.
          \STATE $\Delta \phi_{n,h}\leftarrow 0,$ for all $h$; $\Delta\theta_n\leftarrow0$
          \FOR{k=1, 2, \dots, B}
          \STATE Set $\delta$ to be Eq.~(\ref{eq:actor_gradient}) for the $k^{th}$ element in the batch
          \STATE $\Delta \phi_{n,a_{n,t_k}}\leftarrow\Delta \phi_{n,a_{n,t_k}}+\delta$
          \STATE Set $\delta$ to be Eq.~(\ref{eq:gradlossFunction}) for the $k^{th}$ element in the batch
          \STATE $\Delta\theta_n\leftarrow \Delta\theta_n+\delta$
          \ENDFOR
          \STATE Update $\phi_{n,h}$ by $\Delta \phi_{n,h}$ and $\theta_n$ by $\Delta\theta_n$
          \STATE $\theta'_{n}\leftarrow\tau \theta_{n}+(1-\tau)\theta'_{n}$
          \ENDFOR
        \end{algorithmic}
      \end{algorithm}
   
\section{Simulations} \label{sec:simulation}

In this section, we present our simulation results that evaluate DIP in three different MR-RMB problems. The first two problems are scheduling problems in multi-channel wireless networks with one on minimizing AoI and the other on minimizing holding cost. The third problem considers advertisements placements in social media websites.

We compare the performance of DIP against domain-specific policies in each problem. We also evaluate DeepTOP \cite{nakhleh2022deeptop} in all problems. DeepTOP is a deep online learning algorithm that finds the Whittle index when there is only one kind of resource. In order to incorporate DeepTOP in multi-resource problems, we consider a policy that selects the restless arms with the highest indexes to activate, and then randomly matches selected restless arms to resources. We then train DeepTOP with respect to this policy.

All simulation results are the average of 20 independent runs, with error bars indicating standard deviations. Each run consists of 12,000 time steps. For each time step, we obtain the running average performance of the past 100 time steps. The value of $\vec{\lambda}$ is updated every 100 time steps and we have set the discounted factor to 0.99 (i,e., $\beta=0.99$) and the learning rate of $\vec{\lambda}$ to  0.01 (i.e., $\rho=0.01$) for both problems. The learning rates of neural networks are determined by the ADAM optimizer. {All neural networks of DIP have two fully connected hidden layers with 128 neurons in each hidden layer. We have use the same setting for DeepTOP.}

\subsection{Network Setting}

We introduce the settings of multi-channel wireless networks that will be used in both the AoI minimization problem and the holding cost minimization problem. We consider two types of systems: heterogeneous channels and homogeneous channels. Each type of system has two different settings. In all settings, we consider the challenge that wireless transmissions are not reliable. When a mobile user $n$ is scheduled to transmit on channel $h$, the transmission will be successful with probability $p_{n,h}$.

We have two settings for heterogeneous channels, one with two channels and the other with three channels. For the two-channel system, we assume that there are 20 mobile users. The first 14 users have $p_{n,1}=0.7$ and $p_{n,2}=0.3$. The other 6 users have $p_{n,1}=0.3$ and $p_{n,2}=0.7$. Each channel has a capacity of two, that is, $C_h=2$ for all $h$. For the three-channel system, we assume that there are 34 mobile users. The first 20 users have $p_{n,1}=0.9, p_{n,2}=0.5, p_{n,3}=0.1$, the next 4 users have $p_{n,1}=0.1, p_{n,2}=0.9, p_{n,3}=0.5$, and the last 10 users have $p_{n,1}=0.5, p_{n,2}=0.1, p_{n,2}=0.9$. Each channel has a capacity of two.

We also have two settings for homogeneous channels, one with two channels and the other with three channels. For the two-channel system, we assume that there are 20 mobile users. The first 14 users have $p_{n,1}=p_{n,2}=0.7$. The other 6 users have $p_{n,1}=p_{n,2}=0.3$. Each channel has a capacity of two. For the three-channel system, we assume that there are 34 mobile users. The first 20 users have $p_{n,1}=p_{n,2}=p_{n,3}=0.9$, the next 4 users have $p_{n,1}=p_{n,2}=p_{n,3}=0.7$, and the last 10 users have $p_{n,1}=p_{n,2}=p_{n,3}=0.5$. Each channel has a capacity of two.

We note that the homogeneous channel systems are equivalent to single-channel systems where the capacity of the channel is $HC_h$. Since DeepTOP learns the Whittle index in single-channel systems, we can use the performance of DeepTOP as that by the Whittle index policy in homogeneous channel systems.

\subsection{AoI Minimization}

AoI has gained significant research interests due to its elegance in capturing information freshness. We define the AoI of a mobile user recursively as follows: At time $t=0$ the AoI of the user is 1. In each subsequent time step, the AoI increases by 1 if there is no packet delivery for the user, either because the user is not scheduled or because the transmission fails, and the AoI becomes 1 if there is a packet delivery.

We can model the AoI of user $n$ as a MDP where the state of the user $s_{n,t}$ is its AoI. To ensure a finite state space, we cap the AoI at 20. If user $n$ is not scheduled to any channel, then its AoI will increase by one, and hence 
\begin{equation}\label{eq:prob1}
P_n(s_{n,t+1}=\min\{s+1, 20\}|s_{n,t}=s,a_{n,t}=0)=1.
\end{equation}
On the other hand, if user $n$ is scheduled to transmit on channel $h$, then user $n$ will successfully deliver a packet with probability $p_{n,h}$. Hence, we have
\begin{align}\label{eq:prob2}
P_n(s_{n,t+1}&=\min\{s+1, 20\}|s_{n,t}=s,a_{n,t}=h)=1-p_{n,h}, \\
P_n(s_{n,t+1}&=1|s_{n,t}=s,a_{n,t}=h)=p_{n,h}.
\end{align}
The reward of user $n$ is $R_{n}(s_{n,t},a_{n,t}) = - s_{n,t+1}$. The objective is to minimize the total long-term discounted AoI of all users in the system.

\begin{figure}
    \centering
    \subfigure[Two-channel system]{
        \centering
        \includegraphics[width=0.8\linewidth]{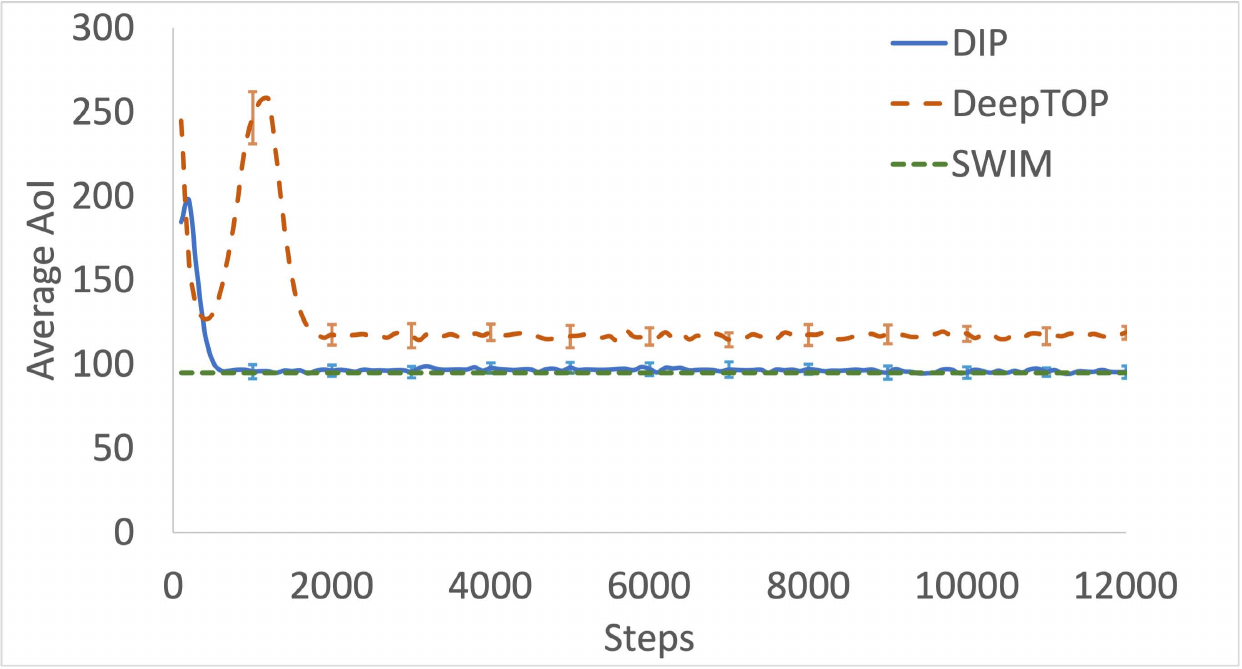}
   
        \label{fig:HetroA20AoI}
    }

    \subfigure[Three-channel system]{
        \centering
        \includegraphics[width=0.8\linewidth]{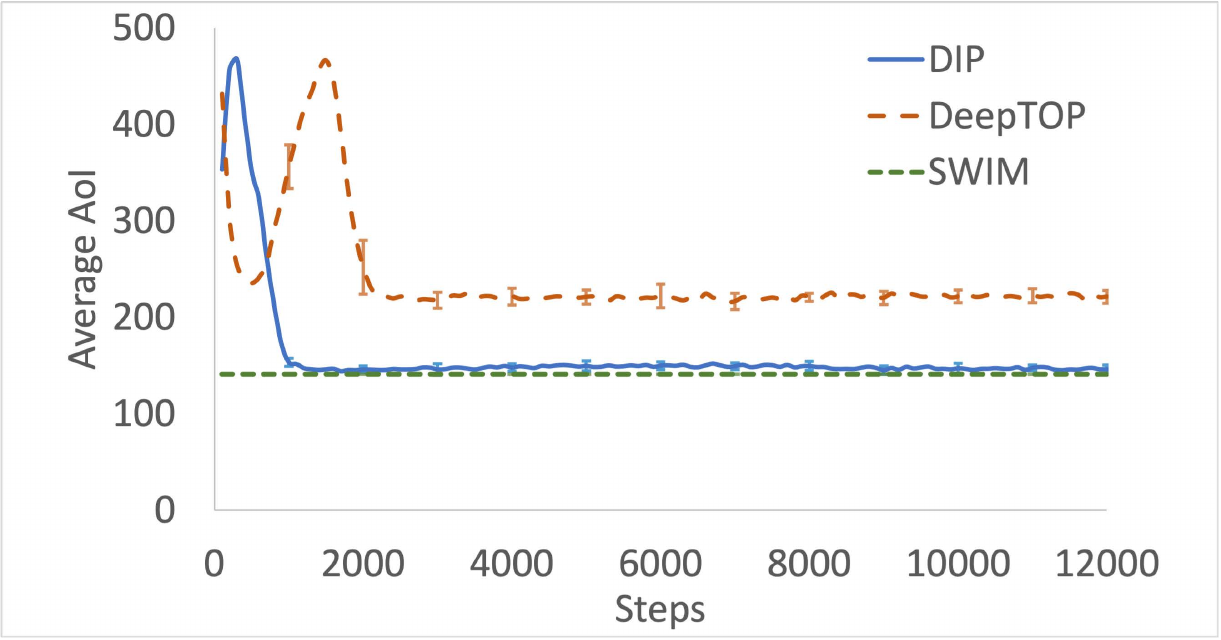}

        \label{fig:HetroA34AoI}
}
    \caption{AoI comparison for multi-channel wireless networks with heterogeneous channels}
    \label{fig:HetroAoI}
\end{figure}

A recent paper \cite{zou2021minimizing} has studied the AoI minimization problem in multi-channel wireless networks. It has proposed a Sum Weighted Index Matching (SWIM) policy. SWIM calculates the {partial index}es of all users on all channels and then use max-weight matching to schedule users. However, SWIM requires the precise knowledge of all $p_{n,h}$ to calculate the {partial index}es. Since our DIP aims to learn the {partial index}es without any prior knowledge of the system, we can use SWIM as a baseline policy.

Fig.~\ref{fig:HetroAoI} shows the simulation results for heterogeneous multi-channel wireless networks. It can be observed that the AoI of DIP converges to the AoI of SWIM in less than 1,000 time steps in both the two-channel system and the three-channel system. This shows that DIP has indeed efficiently learned the {partial index}es and the Lagrange multipliers. It can also be observed that DeepTOP is considerably worse than DIP. This shows the standard Whittle index, which is designed for systems with only one kind of resource, does not work well in multi-resource systems.

Simulation results for homogeneous multi-channel wireless networks are shown in Fig.~\ref{fig:HomogeneousAoI}. As discussed in the previous section, these systems are equivalent to single-channel systems. In such systems, the Whittle index policy is near-optimal. Indeed, the performance of DeepTOP converges to SWIM. DIP also converges to SWIM in less than 2,000 time steps.

\begin{figure}
    \centering
    \subfigure[Two-channel system]{
        \centering
        \includegraphics[width=0.8\linewidth]{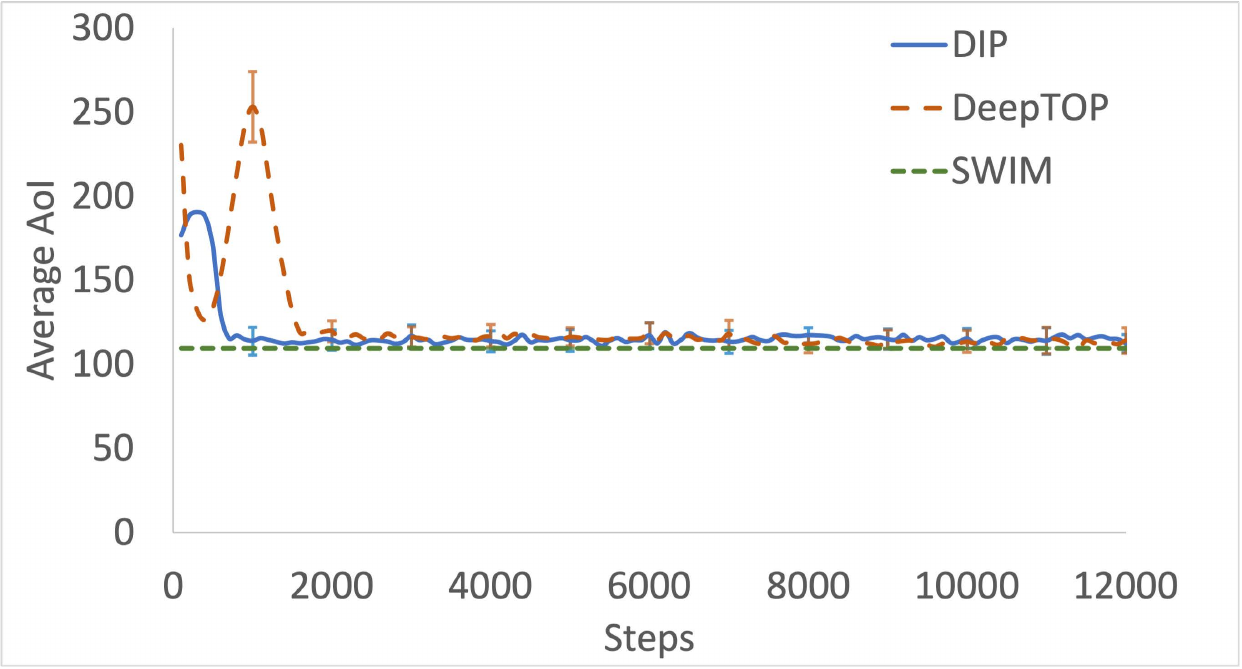}

        \label{fig:HomogeneousA20AoI}
}

    \subfigure[Three-channel system]{
        \centering
        \includegraphics[width=0.8\linewidth]{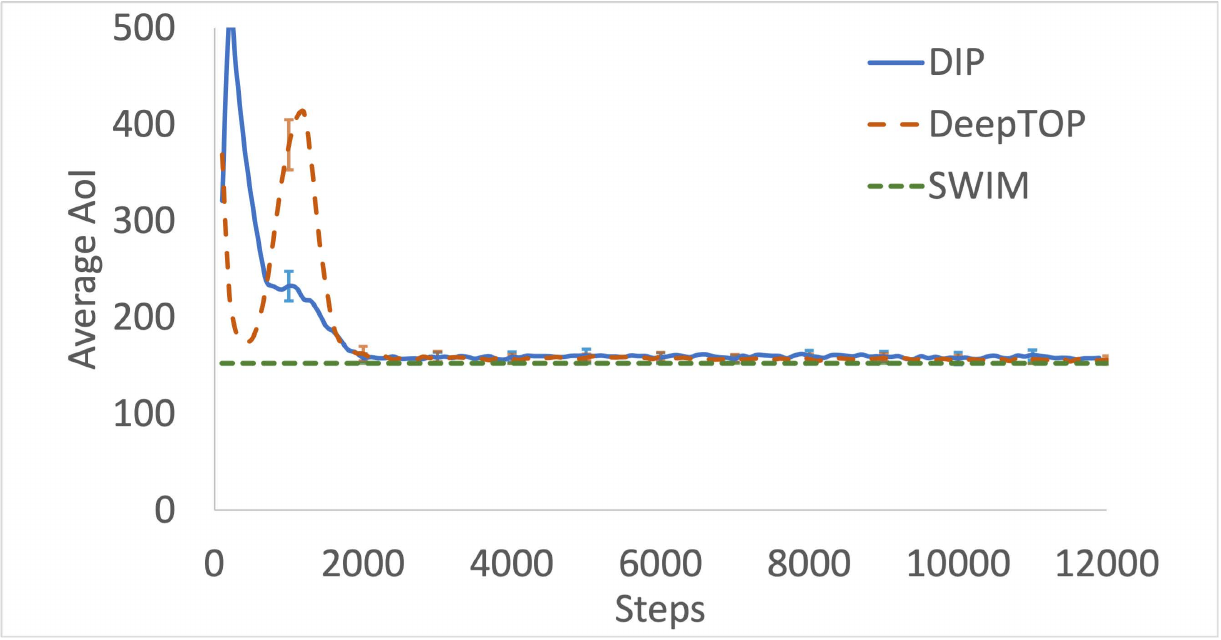}

        \label{fig:HomogeneousA34AoI}
    }
    \caption{AoI comparison for multi-channel wireless networks with homogeneous channels.}
    \label{fig:HomogeneousAoI}
\end{figure}

\subsection{Holding Cost Minimization}

We consider the problem of minimizing holding costs. In this problem, the base station maintains a queue of undelivered packets fro each mobile user. In each time slot, a packet arrives for user $n$ with probability $\zeta_n$. A packet is delivered for user $n$ whenever user $n$ has a successful transmission. In each time slot, each user $n$ incurs a holding cost of its queue size squared.

We can model this problem as a MDP where the state of user $n$, $s_{n,t}$, is its queue size. To ensure a finite state space, we cap the queue size at 20. If user $n$ is not scheduled to any channel, then its queue size will increase if there is a packet arrival, and will remain the same, otherwise. Hence, 
\begin{align}\label{eq:prob1h}
P_n(s_{n,t+1}&=\min\{s+1, 20\}|s_{n,t}=s,a_{n,t}=0)=\zeta_n, \\
P_n(s_{n,t+1}&=s|s_{n,t}=s,a_{n,t}=0)=1-\zeta_n.
\end{align}
If user $n$ is scheduled to channel $h$, then it will have a packet departure with probability $p_{n,h}$. It will have a packet arrival with probability $\zeta_n$. Hence, we have
\begin{align}
P_n(s_{n,t+1}&=\min\{s+1, 20\}|s_{n,t}=s,a_{n,t}=h)=(1-p_{n,h})\zeta_n, \\
P_n(s_{n,t+1}&=s|s_{n,t}=s,a_{n,t}=h)=(1-p_{n,h})(1-\zeta_n)+p_{n,h}\zeta_n,\\
P_n(s_{n,t+1}&=\max\{s-1,0\}|s_{n,t}=s,a_{n,t}=h)=p_{n,h}(1-\zeta_n).
\end{align}
The reward of user $n$ is $R_{n}(s_{n,t},a_{n,t}) = - s_{n,t}^2$.

Ansell et al. \cite{ansell2003whittle} has studied the scheduling problem for holding cost minimization for the special case of single-channel wireless networks. It has derived the Whittle index for this special case. To employ Ansell et al. \cite{ansell2003whittle} for our multi-channel wireless networks, we first assign each user $n$ to a channel $h^*_n$ that has the highest reliability, that is, $h^*_n:=\arg\max_h \{p_{n,h}\}$. We then calculate the Whittle index under this channel assignment, which Ansell et al. has shown to be
\begin{align}\label{eq:anselWhittleIndex}
\frac{3\zeta_n-p_{n,h^*_n}}{p_{n,h^*_n}-\zeta_n}+2p_{n,h^*_n}s_{n,t},
\end{align}
and schedules users according to their Whittle indexes. We call this policy the Whittle index policy.

Fig.~\ref{fig:Hetroholding} shows the simulation results for heterogeneous multi-channel wireless networks, where we set $\zeta_n=0.11$ for both the two-channel system and the three-channel system. It can be observed that DIP significantly outperforms the Whittle index policy and DeepTOP. While the Whittle index policy computes the Whittle index, it can only assign a user to its best channel. On the other hand, DIP allows a user to be matched to the worse channel when its best channel is too congested. The big performance gap between the Whittle index policy and DIP highlights the additional challenges faced in MR-RMB problems.

\begin{figure}
    \centering
    \subfigure[Two-channel system]{
        \centering
        \includegraphics[width=0.8\linewidth]{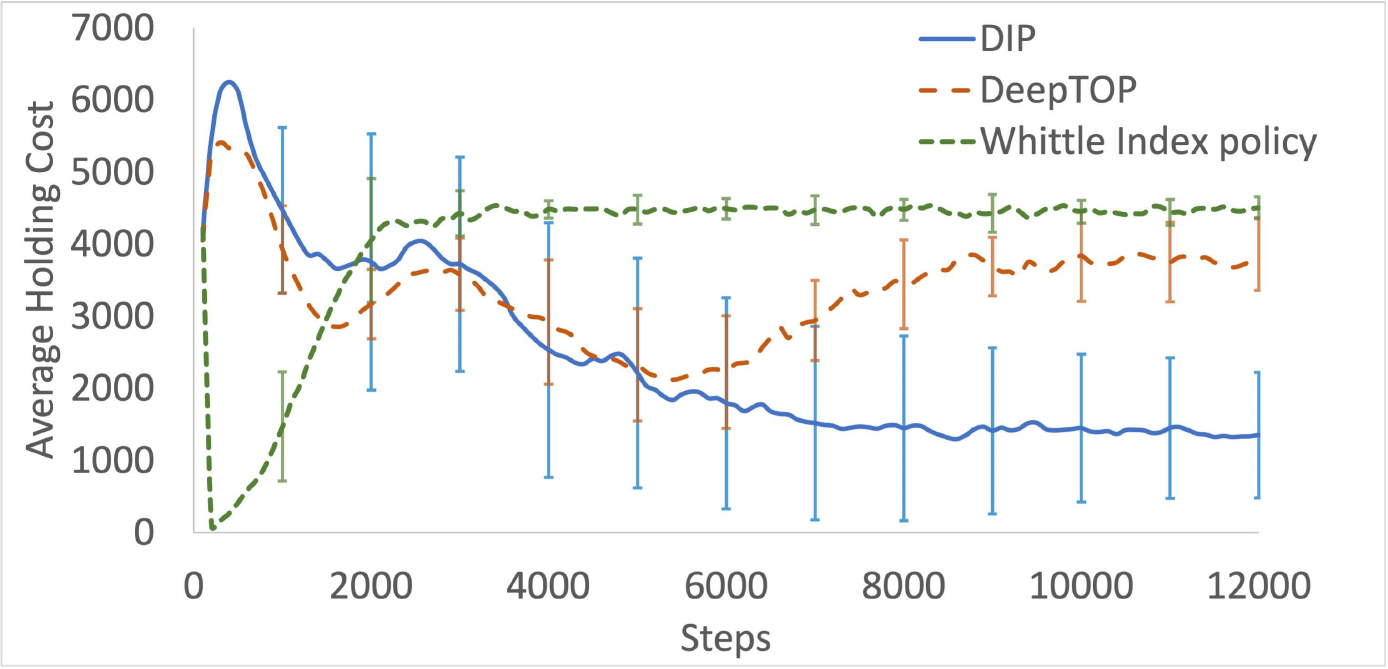}

        \label{fig:HetroA20holding}
   }

\subfigure[Three-channel system]{
        \centering
        \includegraphics[width=0.8\linewidth]{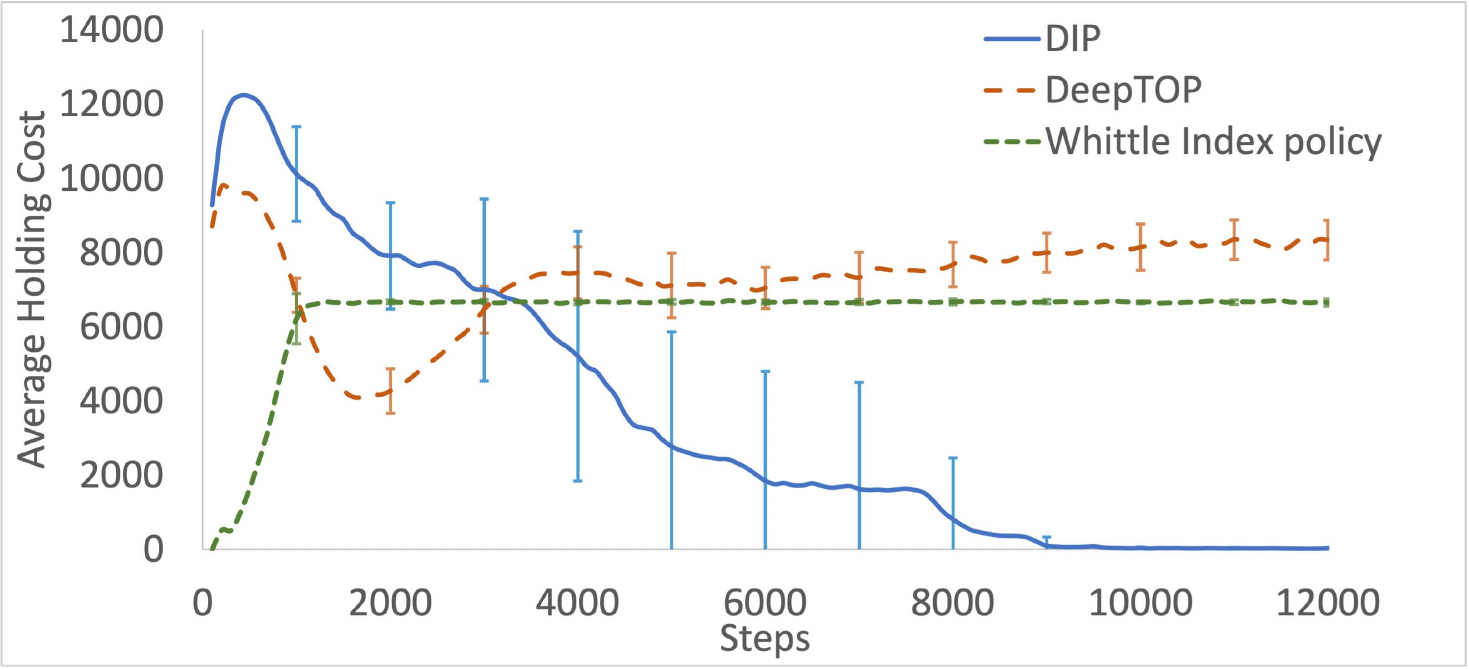}
        \label{fig:HetroA34holding}
}
    \caption{Holding cost comparison for multi-channel wireless networks with heterogeneous channels}
    \label{fig:Hetroholding}
\end{figure}

Simulation results for homogeneous multi-channel wireless networks are shown in Fig.~\ref{fig:Homogeneousholding}. We set $\zeta_n=0.1$ for the two-channel system and $\zeta_n=0.08$ for the three-channel system so that the arrival rates are close to the boundaries of the capacity regions. We observe that all three policies converge to the same holding cost. This is to be expected since these networks are equivalent to single-channel wireless networks. The fact that DIP converges to the Whittle index policy suggests that DIP indeed learns the Whittle index.

\begin{figure}
    \centering
    \subfigure[Two-channel system]{
        \centering
        \includegraphics[width=0.8\linewidth]{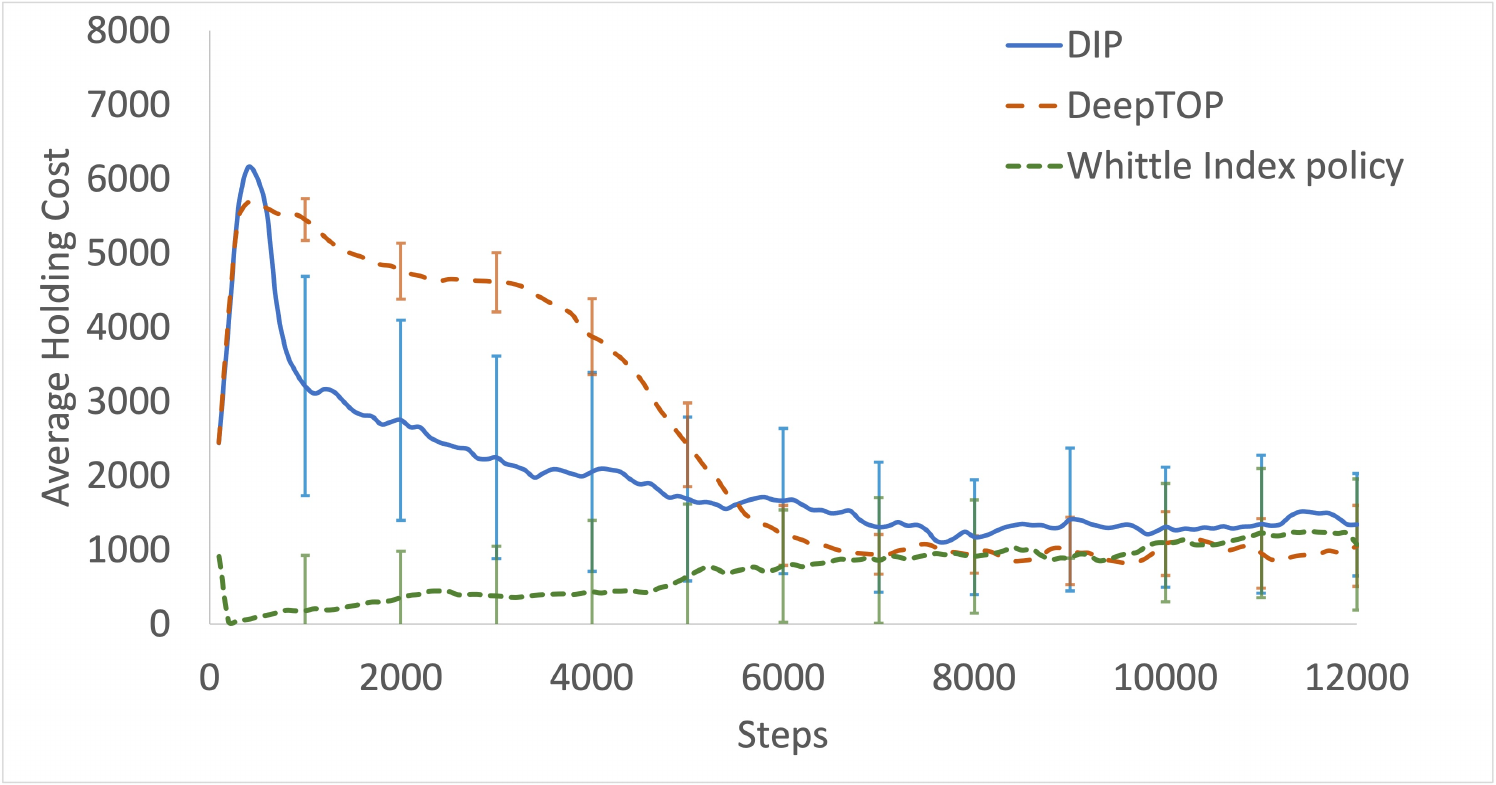}
        \label{fig:HomogeneousA20holding}

}
    \subfigure[Three-channel system]{
        \centering
        \includegraphics[width=0.8\linewidth]{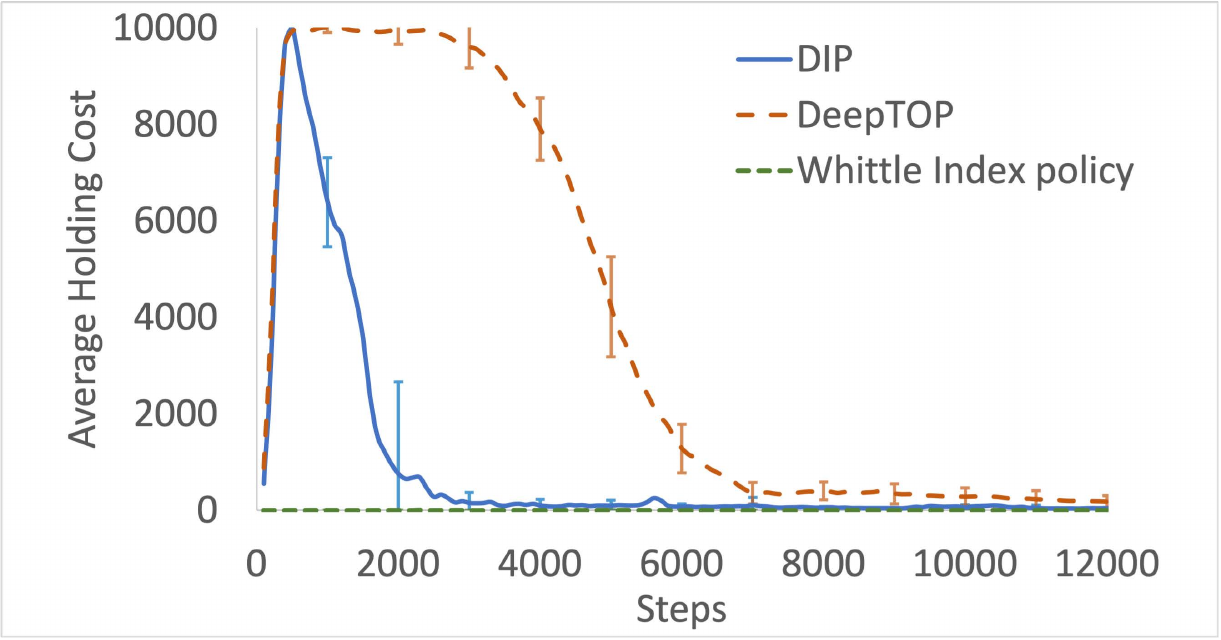}

        \label{fig:HomogeneousA34holding}
 }
    \caption{Holding cost comparison for multi-channel wireless networks with homogeneous channels} 
    \label{fig:Homogeneousholding}
\end{figure}

\subsection{Online Advertisement Placement}
%
We consider an online advertisement placement problem on social media platforms. There are three places where advertisers can display their ads: newsfeed, overhead banner, and sidebar. At each time step, the website administrator determines whether and where to display each advertisement. The consumer interest in a particular advertisement finishes if they click its link and interest may gradually recover over time. For example, someone who recently a purchased a product may not be interested in advertisement for the same product in the immediate future but their interest might gradually revive with time. The objective is to strategically display advertisements in a manner that effectively captures and sustains consumer interest.

We can  model the online advertisement placement problem as a recovering bandit, first introduced by \cite{pike2019recovering}. The objective of recovering bandits is to capture the evolving behaviors of consumers over time. The effectiveness of displaying an advertisement depends on the  elapsed time since the advertisement was last displayed and the placement. At time $t=0$, the elapsed time of all advertisements are set to 1. If the advertisement is not displayed, then elapsed time increased by 1 and if the advertisement is displayed then it is set to the 1.

We formulate recovering bandit as a MDP where each arm of the recovering bandit is the advertisement. The state of the advertisement $s_{n,t}$ is the elapsed time. To ensure the finite state space, we cap the elapsed time at 20. If advertisement $n$ is not displayed to any placement, then its wait time will increase by one, and hence 
\begin{equation}\label{eq:prob1}
P_n(s_{n,t+1}=\min\{s+1, 20\}|s_{n,t}=s,a_{n,t}=0)=1.
\end{equation}
On the other hand, if advertisement $n$ is scheduled to display on placement $h$, then its wait time will set to one, and hence
\begin{equation}\label{eq:prob1}
P_n(s_{n,t+1}=1|s_{n,t}=s,a_{n,t}=h)=1.
\end{equation}
The reward of advertisement $n$ is $R_{n}(s_{n,t},a_{n,t}) = \theta^0_{{n,h}}(1-e^{\theta^1_{{n,h}}.s_{n,t}})$, where $\theta^0_{{n,h}}$. and $\theta^1_{{n,h}}$are the hyperparameters for the placement $h$. The objective is to maximize the total long-term discounted reward of all advertisements in the system.

We consider a setup of 30 advertisements and three advertisement placement  (i.e., $N=30$, $H=3$). Each placement can accommodate up to two advertisements at a time (i.e., $C_h=2$). The hyperparameters of the first 10 advertisements are $\theta^0_{{n,1}}=1, \theta^0_{{n,2}}=3, \theta^0_{{n,3}}=5$,the next 10 advertisements are $\theta^0_{{n,1}}=5, \theta^0_{{n,2}}=1, \theta^0_{{n,3}}=3$ and the last 10 advertisements are $\theta^0_{{n,1}}=3, \theta^0_{{n,2}}=5, \theta^0_{{n,3}}=1$. Additionally, the hypeparameter $\theta^1_{{n,h}}=0.1$   for all advertisements $n$ and all placements  $h$.

Fig.~\ref{fig:HetroAdv} illustrates the simulation results for online advertisement. It is evident that DIP significantly outperforms DeepTOP. Specifically, DIP demonstrates efficient scheduling of advertisements, particularly in situations of high competition for display space. Conversely, DeepTOP tends to prioritize displaying advertisements in less favorable placements when the preferred placements are congested. This indicates that DeepTOP is limited to handling a single type of resource and does not perform effectively in multi-resource systems.
\begin{figure}

\centering
        \includegraphics[width=0.8\linewidth]{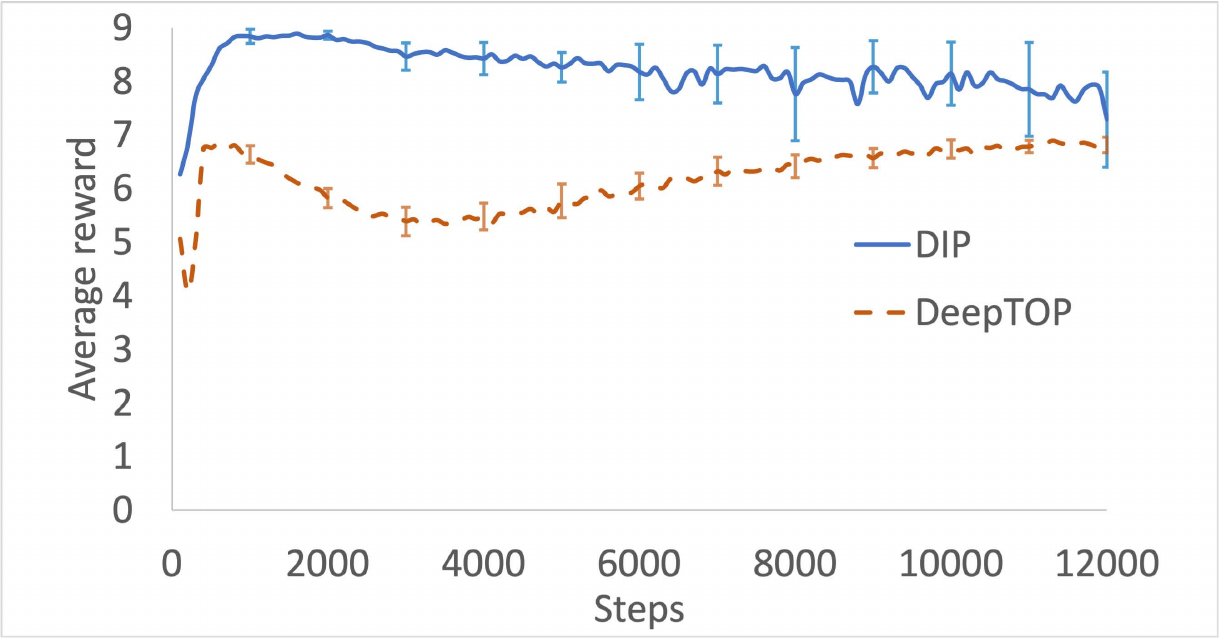}

    \caption{Average reward comparison for online advertisement placement on social media}
    \label{fig:HetroAdv}
\end{figure}

\section{Conclusion} \label{sec:conclusion}
We address the critical challenges of allocating multiple heterogeneous resources in wireless communication systems. We have derived the policy gradient theorem to learn the index. By leveraging policy gradient theorem, we have proposed DIP, an online reinforcement learning algorithm for MR-RMB. Our results show that DIP outperforms existing methods such as DeepTOP and Whittle Index policy, highlighting the limitations of Whittle index based approaches in heterogeneous multi-resource systems. We have also compared DIP algorithm with SWIM and shows that DIP learned {partial index} and dual cost correctly in heterogeneous channels setting. We have also shown that all the results of different policies converge in homogeneous channels setting, signifying that DIP has indeed learned the Whittle index correctly, as homogeneous channels mimic the behavior of a single-channel wireless network. Our findings underscore the versatility of DIP across a wide range of scheduling problems, from homogeneous to heterogeneous resource settings, when transition kernel is unknown and convoluted. We have also generalized DIP to  applications beyond multi-channels. For future study, extending the MR-RMB model to accommodate multiple conflicting objectives could enhance optimization in wireless communication systems. 
\section{Acknowledgement}
This material is based upon work supported in part by NSF under Award Numbers ECCS-2127721 and CCF-2332800 and in part by the U.S. Army Research Laboratory and the U.S. Army Research Office under Grant Number W911NF-22-1-0151.

\bibliographystyle{ACM-Reference-Format}
\bibliography{references}

\appendix
\end{document}